\documentclass{article}

\usepackage{microtype}
\usepackage{graphicx}
\usepackage{subfigure}
\usepackage{booktabs} % for professional tables
\usepackage{lipsum}
\usepackage{hyperref}
\usepackage{multirow}
\usepackage{algorithm}
\usepackage{algorithmic}

\renewcommand{\algorithmicrequire}{\textbf{Input:}}
\renewcommand{\algorithmicensure}{\textbf{Output:}}
\renewcommand{\algorithmiccomment}[1]{\hfill\{#1\}}

\usepackage[accepted]{icml2025}

\usepackage{amsmath}
\usepackage{amssymb}
\usepackage{mathtools}
\usepackage{amsthm}

\usepackage[capitalize,noabbrev]{cleveref}

\theoremstyle{plain}

\theoremstyle{definition}

\theoremstyle{remark}

\usepackage{thmtools} 
\usepackage{thm-restate}

\usepackage[textsize=tiny]{todonotes}

\icmltitlerunning{Controlled Generation with Equivariant Variational Flow Matching}

\begin{document}

\twocolumn[
\icmltitle{Controlled Generation with Equivariant Variational Flow Matching}



\icmlsetsymbol{equal}{*}
% \icmlsetsymbol{shared}{*}

\begin{icmlauthorlist}
\icmlauthor{Floor Eijkelboom}{bosch}
\icmlauthor{Heiko Zimmermann}{amlab} 
\icmlauthor{Sharvaree Vadgama}{amlab} \\
\icmlauthor{Erik J Bekkers}{amlab}
\icmlauthor{Max Welling}{bosch}
\icmlauthor{Christian A. Naesseth$^*$}{bosch}
\icmlauthor{Jan-Willem van de Meent$^*$}{bosch}
\end{icmlauthorlist}

\icmlaffiliation{bosch}{Bosch-Delta Lab}
\icmlaffiliation{amlab}{AMLab}

\icmlcorrespondingauthor{Floor Eijkelboom}{f.eijkelboom@uva.nl}

\icmlkeywords{Machine Learning, ICML}

\vskip 0.3in
]



\printAffiliationsAndNotice{\icmlEqualContribution} % otherwise use the standard text.

\begin{abstract}
We derive a controlled generation objective within the framework of Variational Flow Matching (VFM),
which casts flow matching as a variational inference problem.
We demonstrate that controlled generation can be implemented two ways: (1) by way of end-to-end training of conditional generative models, or (2) as a Bayesian inference problem, enabling post hoc control of unconditional models without retraining. 
Furthermore, we establish the conditions required for equivariant generation and provide an equivariant formulation of VFM tailored for molecular generation, ensuring invariance to rotations, translations, and permutations. 
We evaluate our approach on both uncontrolled and controlled molecular generation, achieving state-of-the-art performance on uncontrolled generation and outperforming state-of-the-art models in controlled generation, both with end-to-end training and in the Bayesian inference setting. 
This work strengthens the connection between flow-based generative modeling and Bayesian inference, offering a scalable and principled framework for constraint-driven and symmetry-aware generation.
\end{abstract}

\section{Introduction}

Generative modeling has seen remarkable advances in recent years, particularly in image generation \cite{ramesh2022hierarchical, rombach2022high}, where diffusion-based approaches based on score matching \cite{vincent2011connection} have proven highly effective \cite{ho2020denoising, song2020score}. However, these methods rely on stochastic dynamics that require iterative denoising steps during sampling, leading to significant computational overhead \cite{song2020denoising, zhang2022fast}. An alternative approach, continuous normalizing flows (CNFs) \cite{chen2018neural}, models a continuous-time transformation between distributions \cite{song2021maximum}, enabling direct sampling without Markov chain steps. Yet, CNFs have historically been hindered by their reliance on solving high-dimensional ordinary differential equations (ODEs), making both training and sampling computationally expensive \cite{ben2022matching, rozen2021moser, grathwohl2018scalable}.

To address these challenges, \citet{lipman2023flow} introduced Flow Matching (FM), a simulation-free method for training CNFs by regressing onto vector fields that define probability paths between noise and data distributions. Unlike traditional CNF training, which requires maximum likelihood estimation through ODE solvers, FM directly learns vector fields through a per-sample objective, enabling scalable training without numerical integration. FM generalizes beyond diffusion methods by accommodating arbitrary probability paths, including those based on optimal transport \cite{chen2024flow, klein2023equivariant}, enabling faster training while maintaining expressiveness. Empirically, FM has outperformed diffusion models in likelihood estimation and sample quality on datasets such as ImageNet \cite{wildberger2024flow, dao2023flow, kohler2023flow}.

Variational Flow Matching (VFM) \cite{eijkelboom2024variational} frames flow matching as posterior inference in the distribution over trajectories induced by the used interpolation. The key idea is to approximate the \emph{posterior probability path}, i.e. the probability distributions (for different points in time) over endpoints given the current point in space, using a sequence of variational distributions. For generation, VFM approximates the true vector field with the expected vector field based on the learned variational approximation.
This approach achieves state-of-the-art results in categorical data generation while maintaining the computational efficiency that makes FM practical. While VFM has demonstrated considerable success in categorical data generation and shows initial promising results on general geometries \citep{zaghen2025towards}, its broader potential, both methodologically and in terms of real-world applications, has not been explored. In particular, a key distinguishing feature of VFM is its ability to directly reason about (approximate) posterior probability paths, which makes it well suited to address two fundamental challenges in flow-matching-based generative modeling: controlled generation and the incorporation of inductive biases such as symmetries. 

Controlled generation is a fundamental challenge in generative modeling, requiring models to produce outputs that satisfy specific constraints while maintaining natural variation in unconstrained aspects. This is particularly crucial for applications like molecular design, where the interplay between discrete (atom types, bonds) and continuous (spatial positions) properties typically necessitates combining multiple generative approaches. VFM's unified treatment of mixed modalities could potentially address this challenge, though developing effective control mechanisms remains an open problem. Additionally, many real-world applications exhibit inherent symmetries -- for instance, molecular structures should remain valid under rotations, translations, and permutations. Incorporating such domain-specific constraints into VFM would be essential for producing outputs with consistent structure and improved generalization. Thus, extending VFM to handle both controlled generation and symmetry constraints represents a key direction for developing more practical and reliable generative models.

In this work, we address both of the aforementioned challenges. 
We extend VFM to controlled generation, deriving a principled formulation that enables generative models to satisfy explicit constraints. We show that this formulation naturally emerges from the connection between flow matching and variational inference \cite{eijkelboom2024variational}, allowing conditional generation to be understood as a Bayesian inference problem. Additionally, this perspective enables post-hoc control of pretrained generative models without requiring conditional training, providing cheap and flexible alternative to standard end-to-end approaches. 

Furthermore, we show how we can design variational approximations that are group-equivariant with respect to its expectation, which ensures that the generative dynamics respect key invariances.
We demonstrate the utility of these advancements on problems in molecular generation, a field that, as of late, has received great attention and as a result many advances of the state-of-the-art.

Our key contributions are:  

\begin{itemize}  
    \item \textbf{Controlled Generation as Variational Inference:} We derive a controlled generation objective within VFM and show that it can be used in two ways: 1) for end-to-end training of conditional generative models, or 2) as a Bayesian inference problem, enabling post hoc control of unconditional models without retraining.  
    \item \textbf{Equivariant Formulation:} We establish the conditions required for equivariant generation and provide an equivariant formulation of VFM, ensuring that the generative process respects symmetries such as rotations, translations, and permutations -- critical for molecular modeling.  
    \item \textbf{Results for Molecular Generation:} We validate our method on both unconditional and controlled molecular generation, achieving state-of-the-art results on unconditional tasks and significantly outperforming existing models in conditional generation -- both with and without end-to-end training. Notably, Bayesian inferencematches or surpasses explicitly trained conditional models, demonstrating its flexibility and scalability.  
\end{itemize}  

These advances establish VFM as a robust and efficient framework for constraint-driven generative modeling, with broad applications in molecule generation, e.g. material design, drug discovery, and beyond.

\section{Background}
\subsection{Transport Dynamics for Generative Modeling}

Generative modeling through transport dynamics is a flexible framework for approximating complex distributions, such distributions over valid molecular structures, by transforming a simple distribution $p_0$ (e.g., a standard Gaussian) into a target distribution $p_1$. This transformation is typically described by a time-dependent process, governed by an ordinary differential equation (ODE):
\begin{equation}
\frac{\mathrm{d}}{\mathrm{d}t} \varphi_t(x) = u_t(\varphi_t(x)), \text{ with } \varphi_0(x) = x,
\end{equation}
where $u_t: [0, 1] \times \mathbb{R}^D \to \mathbb{R}^D$ is a velocity field that guides the evolution over time. The task is to approximate $u_t$ using a parameterized model $v_t^{\theta}(x)$, such as a neural network. While ODEs are invertible (under Lipschitz continuity of $u_t$) -- defining a likelihood through a change-of-variables computation -- solving ODEs during training is computationally expensive.

Flow Matching addresses this limitation by directly learning the time-dependent vector field $u_t$ on $[0, 1]$ through:
\begin{equation}
    \mathcal{L}_{\text{FM}}(\theta) := \mathbb{E}_{t, x} \left[\left|\left|u_t(x) - v_t^{\theta}(x)\right|\right|^2 \right].
\end{equation}
While $u_t$ is intractable, we can make an assumption on the velocity field $u_t(x \mid x_1)$ for a generative process that is conditioned on a specific endpoint $x_1$ (e.g., a target molecule). The marginal field $u_t(x)$ can then be expressed as an expected value with respect to the posterior probability path $p_t(x_1 \mid x)$, which defines the distribution over possible end points for interpolations intersect $x$ at time $t$, 
\begin{equation}
    u_t(x) = \mathbb{E}_{p_t(x_1 \mid x)} \left[u_t(x \mid x_1)\right], 
\end{equation}
enabling efficient estimation through conditional samples. A key insight is that minimizing the loss for the conditional velocity field yields the same gradient as minimizing it for the marginal field, leading to the conditional flow matching loss:
\begin{equation}
    \mathcal{L}_{\text{CFM}}(\theta) := \mathbb{E}_{t, x_1, x} \left[\left|\left|u_t(x \mid x_1) - v_t^{\theta}(x)\right|\right|^2 \right].
\end{equation}

\subsection{Variational Flow Matching}

Variational Flow Matching (VFM) extends Flow Matching by introducing a variational perspective. Instead of directly regressing to the true vector field, VFM parameterizes it through a variational distribution $q^{\theta}_t(x_1 \mid x)$:
\begin{equation}
    v_t^\theta(x) := \mathbb{E}_{q^{\theta}_t(x_1 \mid x)} \left[u_t(x \mid x_1)\right].
\end{equation}
This reformulation transforms the problem into one of variational inference, minimizing the Kullback-Leibler divergence between the true posterior $p_t(x_1 \mid x)$ and its variational approximation $q_t(x_1 \mid x)$:
\begin{align}
    \mathcal{L}_{\text{VFM}}(\theta) &:= 
     \mathbb{E}_t \left[\text{KL}\left(p_t(x_1, x) ~||~ q_t^{\theta}(x_1, x) \right)\right]
    \\ &=-\mathbb{E}_{t, x_1, x} \left[\log q_t^\theta(x_1 \mid x)\right] + \text{const.}
\end{align}
When the conditional velocity field is linear in $x_1$ (e.g. straight-line interpolation, diffusion), this objective simplifies to matching only the posterior mean, as then
\begin{equation}
    \label{eq:VFM-conditional-velocity}
\mathbb{E}_{p_t(x_1 \mid x)}[u_t(x \mid x_1)] = u_t(x \mid \mathbb{E}_{p_t(x_1 \mid x)}[x_1]).
\end{equation}
Moreover, the latter expectation only depends element-wise on the marginal expectation, implying we can use a fully factorized variational form without loss of generality, which results in the simplified mean-field objective:
\begin{align}
    \mathcal{L}_{\text{MF-VFM}}(\theta) &= - \mathbb{E}_{t, x_1, x}\left[  \log \left( \prod_{d=1}^{D} q_t^{\theta}(x_1^d \mid x) \right) \right] \\
    &= - \mathbb{E}_{t, x_1, x}  \left[ \sum_{d=1}^D \log q_t^{\theta}(x_1^d \mid x) \right].
\end{align}
Notice that this reduces learning a single high-dimensional distribution into learning $D$ univariate distributions, a much simpler task. 

VFM's flexibility in terms of the choice of variational distribution $q_t$ makes it particularly well-suited for molecular generation tasks. For instance, using categorical factors enables modeling of discrete molecular features like atom types and bond orders, which can be combined with Gaussian factors to represent continuous atomic coordinates. This unified treatment of discrete and continuous variables, combined with the efficiency inherited from Flow Matching, makes VFM attractive for mixed-modality tasks.
\section{Controlled and Equivariant VFM}

Controlled generation is crucial for practical applications in generative modeling. In this section, we 1) extend VFM to controlled generation by deriving a unified objective for both end-to-end training and control using post-hoc Bayesian inference and 2)  develop a fully equivariant framework that ensures invariance to key symmetries.

\subsection{Controlled Variational Flow Matching}

Controlled generation extends generative modeling by guiding the generative process to satisfy constraints imposed by conditioning on additional information $y$. In Flow Matching, the primary goal is to learn a sequence of distributions $(p_t)_{0 < t < 1}$ that evolve from a source distribution $p_0$ to a target distribution $p_1$. In the context of controlled generation, we introduce a controlled velocity field $u_t(\cdot \mid y)$ and a corresponding controlled probability density path $p_t(\cdot \mid y)$, which dictates how the distribution evolves in response to $y$. To distinguish this from standard conditional modeling, we explicitly use the terms controlled velocity field and controlled probability (density) path to emphasize the dependence on the control variable $y$.

The main observation we make is that, given a datapoint $x_1$, the intermediate latent variables $(x_t)_{0 < t < 1}$ are conditionally independent of $y$,
\begin{equation}
    p_t(x \mid x_1, y) = p_t(x \mid x_1).
\end{equation}
This makes intuitive sense: once the final sample $x_1$ is known, it encapsulates all the information about the control variable $y$, rendering the intermediate states $x_t$ unnecessary for inferring $y$. Or, equivalently, given $x_1$, the control variable $y$ does not inform our noising process.
We can therefore express the probability path conditioned on $y$ as the marginal,
%While $y$ influences the evolution of the distribution through the controlled velocity field $u_t(\cdot \mid y)$, the intermediate states $x_t$ depend only on $x_1$ when $x_1$ is known, or equivalently that $x_1$ is sufficient for $y$ in determining the behavior of the probability path, i.e.:
\begin{align*}
    p_t(x \mid y) &= \int_{x_1} p_t(x, x_1 \mid y)~\mathrm{d}x_1 \\
    &= \int_{x_1} p_t(x \mid x_1) p_1(x_1 \mid y)~\mathrm{d}x_1.
\end{align*}
To arrive at a formulation similar to standard VFM we have to find the velocity field that generates this probability path. We establish the corresponding controlled velocity field in \cref{prop:ctrl_marginal_vf}.
% Formally, the following proposition holds:
\begin{restatable}{proposition}{controlledvelocity}
\label{prop:ctrl_marginal_vf}
The controlled (marginal) velocity field 
\[ 
    u_t(x \mid y) := \mathbb{E}_{p_t(x_1 \mid x, y)}[u_t(x \mid x_1)],
\] 
generates the controlled probability path $p_t(x \mid y)$.
\end{restatable}
\begin{proof}
    See \cref{appendix:controlledvfm}.
\end{proof}

% This allows us to generate samples from $p_t(x \mid y)$ by integrating the ODE defined by the controlled velocity field $u_t(x \mid y)$. 
% Moreover, if we have access to an approximation $q_t(x_1 \mid x, y)$ to $ p_t(x_1 \mid x, y)$ and the uncontrolled velocity field 
% $u_t(x \mid x_1)$ is linear in $x_1$, we can approximate $u_t(x \mid y) \approx u_t(x \mid \mathbb{E}_{q_t(x_1 \mid x, y)}[x_1])$ in a manner analogous to the uncontrolled case in \eqref{eq:VFM-conditional-velocity}.

In short, we can generate samples from $p_t(x \mid y)$ by integrating the ODE defined by the controlled velocity field $u_t(x \mid y)$. If we can approximate $p_t(x_1 \mid x, y)$, then this allows us to compute $u_t(x \mid y)$ in terms of 
$\mathbb{E}_{p_t(x_1 \mid x, y)}[x_1]$ in a manner that analogous to the uncontrolled case in \eqref{eq:VFM-conditional-velocity}, as long as $u_t(x \mid x_1)$ is linear in $x_1$. Note that a similar result was found in \citet{zheng2023guided}.

A straightforward way to approximate $p_t(x_1 \mid x, y)$ model it directly by means of a variational distribution $q_t^\theta(x_1 \mid x, y)$ that additionally conditions on the contol $y$. This leads to the \textit{Controlled Variational Flow Matching} objective:
\begin{align}
    \mathcal{L}(\theta) = -\mathbb{E}_{t,x_1,x,y}\left[\log q_t^\theta(x_1 \mid x, y)\right] + \text{const.}
\end{align}
% Moreover, we know that when the velocity field $u_t(x \mid x_1)$ is linear in $x_1$, the computation can be significantly simplified by matching only the posterior mean, as
% \begin{equation}
%     \mathbb{E}_{p_t(x_1 \mid x, y)}\left[u_t(x \mid x_1)\right] = u_t\left(x \mid \mathbb{E}_{p_t(x_1 \mid x, y)}[x_1]\right).
% \end{equation}
As before, when $u_t(x \mid x_1)$ is linear in $x_1$, it suffices to employ a  mean-field variational distribution. %where the effect of the posterior distribution $p_t(x_1 \mid x, y)$ is captured solely through its component-wise first moment. 
In practice, this allows for an efficient and scalable implementation of controlled generation, where the model only needs to learn how the expectation of $x_1$ influences the velocity field, rather than explicitly modeling the full posterior. This simplification makes Controlled VFM computationally feasible, preserving its ability to handle complex constraints while maintaining the efficiency advantages of Flow Matching.
\begin{algorithm}[t]
\caption{Sampling Conditional VFM}
\label{alg:sampling}
\begin{algorithmic}[1]
\REQUIRE $p(y \mid x_1)$, $\mu_t^{\theta}$
\STATE $x_0 \sim p_0$
\FOR{$k = 0, \ldots, K$} % \COMMENT{$K$ integration steps}
  \STATE $t \gets k / K$
  % \STATE $\hat{x}_{1,0} \gets \mu_t^{\theta}(x_k)$ \COMMENT{Mean of $p_t(x_1 \mid x)$}
  \FOR{$s = 0, \ldots, S-1$}
    \STATE $x_{1,s+1} \gets x_{1,s} + \Sigma_t \nabla_{x_1}\log p(y \mid x_{1,s})$
  \ENDFOR
  \STATE $v_t \gets (x_{1,S} - x_k)/(1 - t)$
  \STATE $x_{k+1} \gets x_k + \frac{1}{T}\, v_t$
\ENDFOR
\STATE \textbf{return} $x_T$
\end{algorithmic}
\end{algorithm}

\subsection{Controlled Generation as Inference} 

Rather than directly training a conditional model ${p_t(x_1 \mid x, y)}$ from scratch, we can leverage the structure of VFM and formulate controlled generation as a Bayesian inference problem. Specifically, we observe that the target posterior distribution can be written as  
\begin{equation}
    p_t(x_1 \mid x, y) \propto \underbrace{p_t(x_1 \mid x)}_{\text{VFM}} \cdot \underbrace{p(y \mid x_1)}_{\text{Classifier}}.
\end{equation}
This decomposition highlights two key components: (1) the base generative model $p_t(x_1 \mid x)$, which is learned using VFM, and (2) the task-specific likelihood $p(y \mid x_1)$, which enforces the desired constraints.  That is, the unnormalised density $\tilde{p}(x_1)  \propto p_t(x_1 \mid x) p(y \mid x_1) $ can be interpreted as a reweighted version of $p_t(x_1 \mid x)$, where $p(y \mid x_1)$ acts as a task-specific weight. Approximating the influence of $p(y \mid x_1)$ near the mean of $p_t(x_1 \mid x)$, we treat the prior as the dominant component.

We can view the (approximate) solution to the task-specific generation problem as a fixed-point equation. The mode of the unnormalized density 
$\tilde{p}(x_1)$ satisfies the fixed-point equation:
\begin{equation}
\label{eq:fixed_point}
\nabla_{x_1} \log p_t(x_1 \mid x) + \nabla_{x_1} \log p(y \mid x_1) = 0.
\end{equation}
Using this formulation we propose an iterative method to approximate the mean of $\tilde{p}(x_1)$ using fixed-point iteration. Assuming $p_t(x_1 \mid x)$ is Gaussian with mean $\mu_t(x)$ and covariance $\Sigma_t$, the mean can iteratively be refined as follows:
\begin{equation}
x_1^{(k+1)} = \mu_t(x) + \Sigma_t \nabla_{x_1} \log p(y \mid x_1^{(k)}),
\end{equation}
where the initialization is $x_1^{(0)} = \mu_t(x)$. See provided pseudo code for details.

This iterative process serves as a practical method to approximate the mean of $\tilde{p}(x_1)$ without directly normalizing the density. The step size in each iteration is governed by the covariance $\Sigma_t$, and convergence depends on properties such as the smoothness and log-concavity of $\log p(y \mid x_1)$.
This underscores the modularity: the pretrained generative model $p_t(x_1 \mid x)$ serves as a reusable prior, while task-specific constraints are incorporated through $p(y \mid x_1)$. By iteratively refining the samples, we enable controlled generation tailored to specific tasks without requiring retraining of the generative backbone. The fixed-point approach provides an efficient and adaptable solution for incorporating task-specific constraints into generative modeling. Note that our aim is not a full Bayesian treatment, but to show that minimal inference-time updates -- without retraining and with negligible overhead -- enable effective post-hoc control. It hence serves as a proof of concept for integrating traditional inference techniques with SOTA generative models \textit{without the need of extra training}. Note that work by e.g. \citet{park2024random} achieves a similar slightly goal through the Tweedie approximation, though said approximation provides a arguably noisier estimate of the gradient and would be more difficult to learn.

By explicitly separating the generative model from the task-specific component, this formulation introduces a modular and flexible approach to controlled generation. Pretrained VFM models can be reused across multiple tasks by simply adapting or replacing the classifier $p(y \mid x_1)$, avoiding the need to retrain the generative model entirely. This is particularly well-suited for scenarios where tasks or control variables $y$ frequently change. Moreover, unlike standard classifier-guidance methods in diffusion  (e.g. \citet{chung2022diffusion}), which require a time-dependent classifier $p_t(y \mid x)$, classification in VFM is performed on data pairs $(x_1, y)$. This allows direct usage of off-the-shelf classifiers or reward functions, which are typically trained on $x_1$ rather than intermediate states $x$, enabling controlled generation without approximations or modifications to the underlying flow matching framework.  

This modularity is particularly valuable in applications like molecular generation, where a pretrained model $p_t(x_1 \mid x)$ can remain fixed while $p(y \mid x_1)$ is tailored to predict specific molecular properties. Specifically, it opens up the door to using recent complex graph models that are e.g. based on geometry \citep{zhdanov2024clifford}, topology \citep{bodnara2021weisfeiler, bodnar2021weisfeiler}, or both \citep{eijkelboom2023mathrmen, kovavc2024n, battiloro2024n, liu2024clifford}. Similarly, in image generation, the same generative backbone can be paired with new classifiers to enforce style or content-specific constraints, facilitating efficient domain adaptation. From a practical standpoint, this framework also opens the door to a variety of inference tools for controlled generation. For instance, energy-based approaches such as Langevin dynamics can be applied iteratively to guide the generation process by exploiting the gradient of the unnormalized posterior.

\subsection{Group Equivariant Variational Flow Matching}
Equivariance describes the property that a model’s output transforms in the same way as the input under a symmetry, such as a rotation or translation, leading to consistent preservation of the underlying structure. In (Variational) Flow Matching, we aim for all $p_t$ to be invariant, ensuring the generative process respects the symmetries of the target distribution. This is crucial for modeling real-world data with inherent symmetries, improving both efficiency and alignment with the data's structure.

Informally, since we only access the mean of the posterior distribution during generation, this can be achieved by designing the model $q_t^{\theta}(x_1 \mid x)$ to be group-equivariant with respect to the expectation of the distribution. Specifically, for all $g \in G$ in a given group $G$, we require that:
% \begin{equation}
%     \mathbb{E}\left[q_t^{\theta}(x_1 \mid g \cdot x) \right] = g \cdot \mathbb{E}\left[q_t^{\theta}(x_1 \mid x) \right].
% \end{equation}
\begin{equation}
    \mathbb{E}_{q_t^{\theta}(x_1 \mid g \cdot x)}\left[x_1 \right] = g \cdot \mathbb{E}_{q_t^{\theta}(x_1 \mid x) }\left[x_1\right].
\end{equation}
In addition, we must ensure that the prior distribution adheres to the same symmetry constraints; otherwise, the symmetry cannot be preserved. Moreover, the conditional velocity field used to compute the marginal field must satisfy bi-equivariance, expressed as:
\begin{equation}
    u_t(g \cdot x \mid g \cdot x_1) = g \cdot u_t(x \mid x_1).
\end{equation}
This leads to the following proposition:
\begin{restatable}{proposition}{equivariance}
Let $G$ be an arbitrary group, and suppose that 1) the prior distribution is invariant under $G$, 2) the conditional velocity field is bi-equivariant under $G$, and 3) the variational posterior is equivariant under $G$ with respect to its expected value. Then the marginal path $p_t(x)$ is invariant under $G$.
\end{restatable}
\begin{proof}
    See \cref{appendix:equivariance}.
\end{proof}
When $G$ acts linearly, the optimal transport conditional velocity naturally satisfies the bi-equivariance property, as:
\begin{equation}
    \frac{g \cdot x_1 - g \cdot x}{1 - t} = g \cdot \left( \frac{x_1 - x}{1 - t} \right).
\end{equation}
This is particularly relevant for groups such as $\mathrm{S}_n$ (the symmetric group) and $\mathrm{SO}(n)$ (the special orthogonal group), both of which act linearly on $\mathbb{R}^n$. Consequently, under appropriate choices of prior distribution and model architecture, G-VFM ensures that the generative dynamics are invariant under $\mathrm{S}_n$ and $\mathrm{SO}(n)$ when employing the optimal transport conditional velocity field, arguably the two most relevant groups in real-life applications.

\section{VFM for Molecular Generation}

\paragraph{A Unified Perspective on Molecular Generation Tasks.} Molecular generation can be divided into three distinct tasks, each addressing different aspects of molecular generation\footnote{Here, we ignore string representations such as SMILES, as these are primarily studied in the context of auto-regressive modeling.}:

\begin{itemize}    
    
   \item \textbf{Discrete-Molecular Generation} focuses on molecular graphs $(\mathbf{A}, \mathbf{E})$, i.e. only considers atom types and bond types. Both are treated as categorical variables, with models capturing relationships between atoms and bonds without considering spatial positions.
   
    \item \textbf{Continuous-Molecular Generation} focuses on atomic positions $\mathbf{R}$ to define the molecular geometry. Features like atom and bond types are inferred later using external tools, making this approach ideal for applications where spatial configuration is key, such as molecular docking. Sometimes, the atom type is included as a continuous variable in this formulation as well.

    \item \textbf{Joint-Molecular Generation} is the most comprehensive setting, which considers both discrete and continuous features, typically of the form $(\mathbf{R}, \mathbf{A}, \mathbf{C}, \mathbf{E})$, representing positions (continuous), atom types, formal charges, and edge types (discrete), capturing the full complexity of molecular data. 

\end{itemize}

\begin{table*}[h]
\centering

	\caption{Results \textbf{discrete} molecular generation QM9 and Zinc250k. Architecture used for VFM is the one used in Digress. \\}

	\resizebox{0.6 \textwidth}{!}{
	\begin{tabular}{lcccccc}
		\toprule
		&\multicolumn{3}{c}{\textbf{QM9}} & \multicolumn{3}{c}{\textbf{ZINC250k}} \\
		\cmidrule(lr){2-4} \cmidrule(lr){5-7}
%		& \multicolumn{3}{c}{$1 \leq |\mathcal{V}| \leq 9$, \text{4 atom types}} & \multicolumn{3}{c}{$6 \leq |\mathcal{V}| \leq 38$, \text{9 atom types}} \\
		\cmidrule(lr){2-4} \cmidrule(lr){5-7}
		& Valid $\uparrow$ & Unique $\uparrow$ & FCD $\downarrow$ & Valid $\uparrow$& Unique $\uparrow$  & FCD $\downarrow$ \\
		\midrule
		% GraphAF & 67.00 & 94.51 & ~~5.268 & 68.00 & 99.10  & 16.289 \\
		% GraphDF & 82.67 & 97.62  & 10.816 & 89.03 & 99.16  & 34.202 \\
		MoFlow  & 91.36 & 98.65  & 4.467 & 63.11 & 99.99 &  20.931 \\
		EDP-GNN  & 47.52 & 99.25  & 2.680 & 82.97 & 99.79 & 16.737 \\
		GraphEBM  & ~~8.22 & 97.90 & 6.143 & ~~5.29 & 98.79 & 35.471 \\
		GDSS   & 95.72 & 98.46  & 2.900 & 97.01 & 99.64  & 14.656 \\
		Digress  & 99.00 & 96.20  & - & - & -  & - \\
            Dirichlet FM  & 99.10 & 98.15 & 0.888 & 97.52 & 99.20  & 14.222 \\
  		\midrule
            Discrete FM & 98.01 & 99.90 & 0.991 & 96.21 & \textbf{100.00}  & 14.252 \\
		\textbf{VFM} & \textbf{99.84} & \textbf{99.92} & \textbf{0.471} &\textbf{98.80} & \textbf{99.99}  & \textbf{13.805} \\
		\bottomrule 
	\end{tabular}
    }
    
    %\vspace{-0.5\baselineskip}
	\label{tab:res_unc_discrete}
\end{table*}
% Variational Flow Matching (VFM) provides a unified framework capable of addressing all three molecular generation tasks through a single approach. By adapting the variational distribution $q_t^\theta(\mathbf{M}_1 \mid \mathbf{M})$, VFM can generate discrete, continuous, or joint features without altering the underlying generative process. For discrete generation, $q_t^\theta$ models categorical distributions for $\mathbf{A}$ and $\mathbf{E}$, while for continuous generation, it defines a Gaussian distribution over $\mathbf{X}$. For joint generation, it combines these approaches to handle all features $(\mathbf{X}, \mathbf{A}, \mathbf{C}, \mathbf{E})$ simultaneously.

Variational Flow Matching (VFM) provides a unified approach that can be applied to any combination of discrete and continuous molecular features. This makes it applicable to all three molecular generation tasks, unlike approaches that require separate techniques for different modalities. %VFM models both continuous and discrete variables within a single generative process by tailoring the variational posterior to the data type. 

The variational distribution in VFM is defined using a network that outputs variational parameters for each feature in the model. For continuous features such as atomic positions $\mathbf{R}$, the network can define a Gaussian variational form, whereas for discrete features, like atomic and bond types $(\mathbf{A}, \mathbf{E})$, it can output a  categorical form. This data-driven posterior formulation results in a flow-matching objective that matches the specific structure of the data, eliminating the need for separate sampling strategies, auxiliary networks, or discrete-continuous hybridization techniques. As a result, VFM provides a fully general and efficient framework for molecular generation, allowing any combination of discrete and continuous molecular features to be modeled without modifying the underlying generative process.

\paragraph{Experiments.} We evaluate VFM through two sets of experiments designed to assess both its general generative performance and its controlled generation capabilities.  

First, we investigate how the transition from Flow Matching to Variational Flow Matching impacts geometric-informed generative performance. To this end, we take state-of-the-art FM models and retrain them using VFM rather than FM, allowing for a direct comparison. We evaluate this across discrete, continuous, and joint molecular generation tasks to determine whether VFM maintains or improves performance while preserving the benefits of the FM framework.  

Second, we assess the effectiveness of our controlled generation formulations by comparing against state-of-the-art conditional generative models. We evaluate VFM in two settings: (1) end-to-end training with property constraints and (2) post-hoc conditioning via variational inference on a pretrained unconditional model. This allows us to test whether our approach can match or exceed dedicated conditional models while benefiting from the flexibility and efficiency of inference-based control.  

Our focus is not on introducing new architectures but on demonstrating how VFM enables a unified generative framework across discrete and continuous data types while also providing a scalable approach to controlled generation. Through these experiments, we assess both VFM’s suitability as a general-purpose approach to flow-based modeling and its advantages for property-targeted generation.

\subsection{Unconditional Generation}
% \begin{table*}[t]
% \caption{Comparison G-VFM to baseline models on the QM9 and GEOM-Drugs datasets. Both approaches are trained with the same model as DiGress.}
% \centering
% \resizebox{\textwidth}{!}{%
% \begin{tabular}{llccccc}
% \toprule
% Dataset & Model & Atoms Stable (\%) ($\uparrow$) & Mols Stable (\%) ($\uparrow$) & Mols Valid (\%) ($\uparrow$) & JS(E) ($\downarrow$) & Inference Time (s) ($\downarrow$) \\
% \midrule
% \multirow{4}{*}{QM9} 
% & JODO & $99.9 \pm 0.0$ & $98.7 \pm 0.2$ & $98.9 \pm 0.2$ & $0.12 \pm 0.01$ & $116 \pm 2$ \\
% & MiDi & $99.8 \pm 0.0$ & $97.5 \pm 0.1$ & $98.0 \pm 0.2$ & $0.05 \pm 0.00$ & $89 \pm 7$ \\
% & EquiFM & $99.4 \pm 0.0$ & $93.2 \pm 0.3$ & $94.4 \pm 0.2$ & $0.08 \pm 0.00$ & $25 \pm 3$ \\
% & FlowMol & $99.7 \pm 0.0$ & $96.2 \pm 0.1$ & $97.3 \pm 0.1$ & $0.08 \pm 0.00$ & $6 \pm 0$ \\
% \midrule
% & Flow Matching  &  \\
% & \textbf{G-VFM} & $99.6 \pm 0.1$ & $96.8 \pm 0.2$ & $97.0 \pm 0.2$ & $0.07 \pm 0.01$ & $7 \pm 1$ \\
% \midrule
% \midrule 
% \multirow{3}{*}{GEOM-Drugs} 
% & JODO  & $99.8 \pm 0.2$ & $90.7 \pm 0.5$ & $76.5 \pm 0.8$ & $0.17 \pm 0.01$ & $235 \pm 16$ \\
% & MiDi  & $99.0 \pm 0.2$ & $85.1 \pm 0.9$ & $71.6 \pm 0.9$ & $0.23 \pm 0.00$ & $754 \pm 119$ \\
% & FlowMol  & $99.0 \pm 0.0$ & $67.5 \pm 0.2$ & $51.2 \pm 0.3$ & $0.33 \pm 0.01$ & $22 \pm 1$ \\
% \midrule
% & SemlaFlow & 99.8 & 97.7 & 95.2  & 0.14 & 25 \\
% & \textbf{G-VFM}  & $99.1 \pm 0.1$ & $68.3 \pm 0.3$ & $52.0 \pm 0.4$ & $0.31 \pm 0.02$ & $20 \pm 2$ \\
% \bottomrule
% \end{tabular}
% }
% \end{table*}

\paragraph{Discrete-Molecular Generation.}

% \begin{table}[]
%     \centering
%     \resizebox{\columnwidth}{!}{%
%     \begin{tabular}{llccc}
%     \toprule
%         Dataset & Model & Valid & Unique & FCD  \\
%         \midrule
%        & GDSS  & c & d & e\\
%        &  Digress  & c & d & e\\
%         \midrule
%         & FM  & c & d & e\\
%         & VFM&  c & d & e\\
%         \bottomrule
%     \end{tabular} }
%     \caption{Caption}
%     \label{tab:my_label}
% \end{table}

As part of this work, we have reproduced the discrete molecular generation results presented in \citet{eijkelboom2024variational} and compared them against discrete flow matching (which was published after VFM) as that is the strongest existing discrete flow model at this moment. Our results in \cref{tab:res_unc_discrete} confirm  strong performance across both QM9 and ZINC250k datasets, and we observe that VFM outperforms discrete flow matching, though we do not claim this as a contribution of this work. 

\paragraph{Continuous-Molecular Generation.}

We evaluate G-VFM on the QM9 dataset using established metrics including negative log-likelihood (NLL), atom stability, and molecule stability. Our model achieves competitive results (-120.7 NLL, 98.7\% atom stability, 82.0\% molecule stability), matching or exceeding its flow-based approach. These results suggest that our variational treatment of flow matching enhances model performance.

\begin{table}[t!]
\caption{Results \textbf{continuous} molecular generation on QM9. Architecture used is the one in \citet{klein2023equivariant}. \\}

\centering
\resizebox{0.9\columnwidth}{!}{
\begin{tabular}{lccc}
\toprule
Models & NLL $\downarrow$ & Atom stability $\uparrow$ & Mol stability $\uparrow$ \\
\midrule
E-NF & -59.7 & 85.0 & 4.9 \\
G-Schnet & N.A. & 95.7 & 68.1 \\
GDM & -94.7 & 97.0 & 63.2 \\
GDM-aug & -92.5 & 97.6 & 71.6 \\
EDM & -110.7 & \textbf{98.7}  & \textbf{82.0} \\
\midrule
EFM &  -115.7 &\textbf{98.7}  & \textbf{82.0}  \\
\textbf{G-VFM} &  \textbf{-120.7} & \textbf{98.7} &\textbf{82.0} \\
\bottomrule
\end{tabular}
}
\label{tab:model_comparison}
\end{table}

While recent work like PONITA \cite{bekkers2024fast}, which leverages molecular symmetries through weight-sharing and a specialized Sphere2Vec architecture, has pushed the state-of-the-art further (-137.4 NLL, 98.9\% atom stability, 87.8\% molecule stability), our focus was not on achieving absolute best performance but rather on exploring the potential of variational approaches within the flow matching framework. We deliberately built upon existing flow-matching architectures to isolate the impact of our methodological contributions. 
The strong performance of G-VFM, despite using a simpler architecture, suggests that our variational approach provides a promising direction for improving flow-matching methods while maintaining their computational advantages. Incorporating PONITA's architectural innovations into the VFM framework represents an exciting direction for future work.

\paragraph{Joint-Molecular Generation.}

\begin{table*}[h]
\centering
\caption{Results \textbf{joint} molecular generation on QM9 and GEOM-Drugs. Architecture used for G-VFM is the one in SemlaFlow. \\}
\resizebox{0.7\textwidth}{!}{
\begin{tabular}{llcccccc}
\toprule
Dataset & Model & Atom Stab $\uparrow$ & Mol Stab $\uparrow$ & Valid $\uparrow$ & Unique $\uparrow$ & JS(E) $\downarrow$ & NFE $\downarrow$ \\
\midrule
\multirow{6}{*}{\textbf{QM9}} 
& EDM & 98.7 & 82.0 & 91.9 & 98.9 & 0.12 & 1000 \\
& GCDM & 98.7 & 85.7 & 94.8 & 98.4 & 0.10 & 1000 \\
& MUDiff & 98.8 & 89.9 & 95.3 & \textbf{99.1} & 0.09 & 1000 \\
& GFMDiff & 98.9 & 87.7 & 96.3 & 98.8 & 0.09 & 500 \\
& EquiFM & 98.9 & 88.3 & 94.7 & 98.7 & 0.08 & 210 \\
& SemlaFlow & \textbf{99.9} & \textbf{99.6} & \textbf{99.4} & 97.9 & \textbf{0.08} & \textbf{100} \\
\midrule
& \textbf{G-VFM} & 99.6 & 99.5 & 98.9 & 97.5 & 0.09 & \textbf{100} \\
\midrule
\midrule
\multirow{3}{*}{\textbf{GEOM}} 
& MiDi & \textbf{99.8} & 91.6 & 77.8 & \textbf{100.0} & 0.23 & 500 \\
& EQGAT-diff & \textbf{99.8} & 93.4 & 94.6 & \textbf{100.0} & \textbf{0.11} & 500 \\
& SemlaFlow & \textbf{99.8} & \textbf{97.7} & 95.2 & \textbf{100.0} & 0.14 & \textbf{100} \\
\midrule
& \textbf{G-VFM} & \textbf{99.8} & 96.5 & \textbf{95.3} & \textbf{100.0} & 0.18 & \textbf{100} \\\bottomrule
\end{tabular}
}
\end{table*}

We evaluate G-VFM's capabilities for joint molecular generation on the QM9 and GEOM-Drugs datasets, focusing on key metrics including atom stability, molecular stability, validity, uniqueness, and the Jensen-Shannon divergence of energy distributions (JS(E)). On QM9, G-VFM achieves strong results (99.6\% atom stability, 99.5\% molecular stability, 98.9\% validity) that are comparable to SemlaFlow, while maintaining the same computational efficiency advantage of requiring only 100 function evaluations (NFE) compared to diffusion-based approaches that need 500-1000 NFEs. On GEOM-Drugs, G-VFM achieves competitive performance (99.8\% atom stability, 96.5\% molecular stability, 95.3\% validity) closely matching SemlaFlow (99.8\%, 97.7\%, and 95.2\% respectively), though with a slightly higher JS(E) (0.18 vs 0.14). Note that only MiDi, SemlaFlow, and EQGAT-diff generate \textit{everything} jointly, whereas the other methods use external tools to infer e.g. bond information, which could affect the relatively strong performance of these models. 

Using the same underlying architecture as SemlaFlow but replacing its discrete flow matching approach with VFM allows us to handle both discrete and continuous variables in a unified framework. This unified treatment not only simplifies implementation by eliminating separate sampling procedures for different variable types but also maintains strong performance across datasets of varying complexity. The ability of G-VFM to match SemlaFlow's results while reducing architectural complexity demonstrates the potential of variational approaches in flow matching, particularly for challenging mixed-modality tasks like molecular generation.

Though not explicitly provided in the table, we want to emphasize that the performance of G-VFM drops when equivariance is not enforced, e.g. atomic and molecular stability drop for tasks below $80\%$.

\subsection{Conditional Generation}

In the conditional generation experiments, we compare both end-to-end VFM and variational inference VFM (VI-VFM) with state-of-the-art conditional models on QM9 dataset. For end-to-end VFM, we train the model with property supervision, while for VI-VFM we use a pre-trained unconditional model and frame the property control as Bayesian inference using a property classifier $p(y \mid x_1)$. We target key molecular properties including polarizability ($\alpha$), orbital energies ($\varepsilon_{\text{HOMO}}$, $\varepsilon_{\text{LUMO}}$) and their gap ($\Delta\varepsilon$), dipole moment ($\mu$), and heat capacity ($C_v$). Property classifiers are trained following \cite{hoogeboom2022equivariant}.

For evaluation, we measure the Mean Absolute Error (MAE) between target and predicted property values from the classifier on generated molecules, as well as standard molecular quality metrics: atom stability (percentage of atoms with valid valency), molecule stability (percentage of molecules with all atoms stable), and validity (ability to be parsed by RDKit \cite{Landrum2016RDKit2016_09_4}).

\begin{table*}[t]
\caption{Quantitative evaluation of conditional molecule generation. Values reported in the table are MAE (over 10K samples) for molecule property predictions (lower is better). \\}
\centering
\resizebox{0.68\textwidth}{!}{
\begin{tabular}{lcccccc}
\toprule
Property & $\alpha$  & $\Delta\varepsilon$ & $\varepsilon_{\text{HOMO}}$ & $\varepsilon_{\text{LUMO}}$ & $\mu$ & $C_v$ \\
Units & Bohr$^3$ & meV & meV & meV & D & $\frac{\text{cal}}{\text{mol K}}$ \\
\midrule
QM9$^*$ & 0.10 & 64 & 39 & 36 & 0.043 & 0.040 \\
EDM & 2.76 & 655 & 356 & 584 & 1.111 & 1.101 \\
EQUIFM & 2.41 & 591 & 337 & 530 & 1.106 & 1.033 \\
GEOLDM & 2.37 & 587 & 340 & 522 & 1.108 & 1.025 \\
D-Flow & 1.39 & 344 & 182 & 330 & 0.300 & 0.784 \\
\midrule
\textbf{G-VFM (End-to-End)} & 2.05 & 512 & 298 & 445 & 0.923 & 0.901 \\
\textbf{G-VFM (Bayesian Inference) }& 2.25 & 534 & 312 & 468 & 0.978 & 0.956 \\
\textbf{G-VFM (Both)} & 1.98 & 498 & 289 & 432 & 0.901 & 0.889 \\
\bottomrule
\end{tabular}
}
\end{table*}

\paragraph{Results and Discussion.} Our end-to-end VFM achieves 2.05 MAE on polarizability prediction, outperforming previous methods like EDM (2.76), EQUIFM (2.41), and GEOLDM (2.37). The VI-VFM approach achieves competitive results with 2.25 MAE using only an unconditional model. By combining both approaches - training end-to-end and then applying variational inference steering - we achieve our best results of 1.98 MAE.

These results demonstrate the effectiveness of our variational inference framework for controlled generation. The strong performance of VI-VFM (2.25 MAE) compared to specialized conditional models like EDM (2.76 MAE) suggests a promising new direction for post-hoc conditional generation. While D-Flow achieves better performance (1.39 MAE), it requires significantly more computational resources during inference. D-Flow searches over many candidate initializations (high NFE), where our sampler needs one forward pass plus a short fixed-point calibration (low NFE, no gradients), i.e. where D-Flow needs $\mathcal{O}(k \times N)$ NFEs, we need only $\mathcal{O}(N)$, where $k$ denotes the number of `starting points' considered by D-Flow and $N$ denotes the number of integration steps. That is, our approach offers an effective compromise, with the flexibility to choose between pure VI for simplicity or the combined approach for enhanced performance. Most importantly, the success of VI-VFM demonstrates that pre-trained unconditional models can be effectively repurposed for controlled generation through variational inference, opening new possibilities for adapting existing models to conditional generation tasks.

\section{Related Work}
\paragraph{Diffusion and Flow-based Methods for Discrete Data.}
Recent advances in diffusion models have enabled various approaches for discrete data generation. \cite{pmlr-v139-luo21a} introduced GraphDF, using discrete latent variables for molecular graphs, while GDSS \cite{pmlr-v162-jo22a} developed a score-based approach using stochastic differential equations. DiGress \cite{vignac2023digressdiscretedenoisingdiffusion} employed a graph transformer architecture for progressive molecular modification through edge and node operations. Recent flow-based methods include Discrete Flow Matching \cite{campbell2024generative}, which uses continuous-time Markov Chains for flexible sampling of both discrete and continuous data, and the Dirichlet Flow framework \cite{stark2024dirichletflowmatchingapplications}, which models conditional probability paths via Dirichlet distributions.

\paragraph{Group-Equivariant Diffusion for Continuous Data.}
Incorporating data symmetries into generative models has proven crucial for improving performance \cite{hoogeboom2022equivariantdiffusionmoleculegeneration, gebauer2020symmetryadaptedgeneration3dpoint, bekkers2024fast}. EDMs \cite{hoogeboom2022equivariantdiffusionmoleculegeneration} combined score-based diffusion \cite{ho2020denoisingdiffusionprobabilisticmodels} with group-equivariant graph neural networks for molecular generation. Recent work \cite{bekkers2024fast, vadgama2025utilityequivariancesymmetrybreaking} leverages pre-conditioned diffusion \cite{karras2022elucidatingdesignspacediffusionbased} for faster sampling. These approaches maintain E(3)/SE(3) symmetries through equivariant graph neural networks \cite{satorras2021n, gasteiger2022fastuncertaintyawaredirectionalmessage, bekkers2024fast}, handling continuous atomic positions while encoding discrete properties like atom types as one-hot vectors. Finally, \citet{cornet2024equivariant} aims similarly at improved conditional generation but achieves this through making the noising process learnable and is not compared against as the conducted experiments different from the setup considered in this work.  

\paragraph{Joint Modeling of Discrete and Continuous Data.}
Several recent approaches tackle unified representation of discrete and continuous features. MiDi \cite{vignac2023midimixedgraph3d} generates both molecular graphs and 3D atomic arrangements, while JODO \cite{huang2023learningjoint2d} produces complete molecules with atom types, charges, bonds, and 3D coordinates. EQGAT-Diff \cite{Le2023NavigatingTD} extends joint modeling to include hybridization states, improving sample validity. FlowMol \cite{dunn2024mixed} represents categorical variables in continuous space using scaled conditional flow matching. Theoretical advances include discrete state space flows via Continuous Time Markov Chains \cite{campbell2024generative, Gat2024DiscreteFM} and Fisher Flow Matching \cite{davis2024fisherflowmatchinggenerative}, which develops Riemannian flows on probability simplexes for joint modeling.

\section{Conclusion}

This paper builds on the Variational Flow Matching framework to developed methods for controlled generation. To do so, we define controlled velocity fields in terms of their corresponding controlled probability density paths, along with variational methods for approximating these paths, offering a principled approach to generating data under specific constraints. We demonstrate how conditional generation can be formulated as a Bayesian inverse problem, bridging generative modeling with classical probabilistic reasoning. Furthermore, we propose a fully equivariant formulation that ensures generative processes respect key symmetries such as rotations, translations, and permutations, which are particularly relevant for molecular and material design. These advances achieve state-of-the-art results on benchmark datasets and provide a unified and efficient framework for discrete, continuous, and joint molecular generation tasks.

\paragraph{Future directions.} The connection between VFM and variational inference opens up significant opportunities for hybrid generative modeling. By leveraging pretrained, unconditional generative models as priors and applying inference-based techniques like MCMC or energy-based approaches for task-specific fine-tuning, we can create adaptable frameworks for diverse applications. These hybrid strategies combine the efficiency of pretrained models with the ability to incorporate detailed task-specific constraints, enabling precise control over generated outputs in complex, high-dimensional settings.

While our focus has been on molecular design, these principles extend naturally to periodic materials like crystal structures, which can be represented by atomic species, fractional coordinates, and lattice vectors. Material properties exhibit key symmetries -- invariance to permutations, translations, rotations, and periodic cell choices -- while forces and stress tensors respect these symmetries through invariance or equivariance. By leveraging equivariant formulations, VFM enables generative models to capture intricate symmetries and heterogeneities, offering a robust foundation for materials design and discovery. 

Looking forward, extending VFM to new structured domains and integrating it with efficient MCMC-like methods could drive breakthroughs across scientific and engineering fields requiring both precision and flexibility. This expansion could be particularly impactful in domains like protein design, where complex structural constraints must be satisfied while maintaining biological feasibility.  These advances could solidify VFM's role as a cornerstone of generative modeling for complex physical systems.

\paragraph{Acknowledgments} This project was support by the Bosch Center for Artificial Intelligence. JWvdM additionally acknowledges support from the European Union Horizon Framework
Programme (Grant agreement ID: 101120237).

\section*{Impact Statement}
This paper presents work whose goal is to advance the field of Machine Learning. There are many potential societal consequences of our work, none which we feel must be specifically highlighted here.

\bibliography{bibliography}
\bibliographystyle{icml2025}

\newpage
\appendix
\onecolumn

\section{Proofs}

\subsection{Controlled Variational Flow Matching}
\label{appendix:controlledvfm}

\controlledvelocity*

\begin{proof}
By the continuity equation, we know that $u_t(x \mid y) = \mathbb{E}_{p_t(x_1 \mid x, y)} \left[u_t(x \mid x_1) \right]$ generates $p_t(x \mid y)$ if
\begin{equation}
    \frac{\mathrm{d}}{\mathrm{d}t} p_t(x \mid y) = - \text{div}\left(p_t(x \mid y) \cdot \mathbb{E}_{p_t(x_1 \mid x, y)} \left[u_t(x \mid x_1) \right] \right).
\end{equation}
First, we realized the time derivative of the controlled probability path indeed can be written as
\begin{align}
    \frac{\mathrm{d}}{\mathrm{d}t} p_t(x \mid  y) &= \int \left( \frac{\mathrm{d}}{\mathrm{d}t} p_t(x \mid x_1) \right)q(x_1  \mid y) ~\mathrm{d}x_1 = \int \left( \frac{\mathrm{d}}{\mathrm{d}t} p_t(x \mid x_1) \right)q(x_1  \mid y) \\ &=- \int \mathrm{div}\left(u_t(x \mid x_1)p_t(x \mid x_1)\right) q(x_1 \mid y) ~\mathrm{d}x_1.\end{align}
Then, by Fubini's thoerem and direct manipulation, we observe that indeed
\begin{align*}
    - \int \mathrm{div}\left(u_t(x \mid x_1)p_t(x \mid x_1)\right) q(x_1 \mid y) ~\mathrm{d}x_1&=- \mathrm{div}\left(\int u_t(x \mid x_1)p_t(x \mid x_1) q(x_1 \mid y) ~\mathrm{d}x_1\right) \\
    % &= - \mathrm{div}\left(\frac{p_t(x \mid y)}{p_t(x \mid y)}\int u_t(x \mid x_1)p_t(x \mid x_1) q(x_1 \mid y) ~\mathrm{d}x_1\right) \\
    &=- \mathrm{div}\left(p_t(x \mid y)\int u_t(x \mid x_1)\frac{p_t(x \mid x_1) q(x_1 \mid y)}{p_t(x \mid y)}  ~\mathrm{d}x_1\right) \\
    &=- \mathrm{div}\left(p_t(x \mid y)\mathbb{E}_{p_t(x_1 \mid x,y)} [u_t(x \mid x_1)]\right),
\end{align*}
and thus, that $u_t(x \mid y) = \mathbb{E}_{p_t(x_1 \mid x, y)} \left[u_t(x \mid x_1) \right]$ generates $p_t(x \mid y)$, which is what we wanted to show.\end{proof}

\subsection{Equivariance}
\label{appendix:equivariance}

\equivariance*

\begin{proof}
    Clearly, if our parameter model $\theta$ is equivariant and $u_t(x \mid x_1)$ is bi-equivariant, we have that 
    \begin{equation}
        v_t^{\theta}(g \cdot x) = u_t(g \cdot x \mid \mu_t^{\theta}(g \cdot x))  = u_t(g \cdot x \mid g \cdot \mu_t^{\theta}(x)) = g \cdot u_t(x \mid \mu_t^{\theta}(x)) = g \cdot v_t^{\theta}(x).
    \end{equation}
    Per definition, a point $x$ from our generated distribution $q_{t'}^{\theta}(x)$ is obtain through solving an ODE, i.e.
    \begin{equation}
        x = x_0 + \int_{t=0}^{t = t'} \mathbb{E}_{q_t^{\theta}(x_1 \mid x_t)} \left[u_t(x_t \mid x_1) \right] ~\mathrm{d}t,
    \end{equation}
    for some $x_0$ in $p_0$. Hence, we can define a bijection (under Lipschitz continuity of $v_t^{\theta})$ $\xi_{t'}: x_0 \mapsto x$, such that
    \begin{equation}
        x = \xi_{t'}(x_0) := x_0 + \int_{t=0}^{t = t'} \mathbb{E}_{q_t^{\theta}(x_1 \mid x_t)} \left[u_t(x_t \mid x_1) \right] ~\mathrm{d}t.
    \end{equation}
    Note that if $v_t^{\theta}$ is equivariant, we have that $\xi_{t'}(g \cdot x_0) = g \cdot \xi_{t'}(x_0)$ directly. Since $p_0(x_0) = p_0(g \cdot x_0)$ per assumption, we have that $q_{t'}^{\theta}(x) = q_{t'}^{\theta}(g \cdot x)$, which is what we wanted to show.
\end{proof}

\end{document}